\documentclass{amsart}
\usepackage{amssymb}
\usepackage{amsmath}
\usepackage{amsfonts, latexsym}
\usepackage{color,graphics}
\usepackage{graphicx}
\usepackage{cite}
\usepackage{caption}
\usepackage{subcaption}
\usepackage{float}
\usepackage[utf8]{inputenc}
\graphicspath{ {./} }
\usepackage{wrapfig}

\usepackage[ruled,vlined]{algorithm2e}

\newtheorem{theorem}{Theorem}[section]

\newtheorem{lemma}[theorem]{Lemma}

\newtheorem{definition}[theorem]{Definition}

\begin{document}

\author{Gurpreet S Kalsi}
\address{The University of Michigan}
\email{gskalsi@umich.edu}

\author{Steve Damelin (Project Advisor)} 
\address{The University of Michigan}
\email{damelin@umich.edu}

\title[Preprocessing power weighted shortest path data with  Well Separated Pair Decomposition]{Preprocessing power weighted shortest path data using a s-Well Separated Pair Decomposition}
\date{\today}


\begin{abstract}
For $s$ $>$ 0, we consider an algorithm that computes all $s$-well separated pairs in certain point sets in $\mathbb{R}^{n}$, $n$ $>1$. 
For an integer $K$ $>1$, we also consider an algorithm that is a permutation of Dijkstra's algorithm, that computes  $K$-nearest neighbors using a certain power weighted shortest path metric in $\mathbb{R}^{n}$, $n$ $>$ $1$.   
We describe each algorithm and their respective dependencies on the input data.  We introduce a way to combine both algorithms into a fused algorithm.  Several open problems are given for future research.
\end{abstract}

\maketitle
\tableofcontents
\newpage
\setcounter{equation}{0}

\section{Introduction}
Clustering high dimensional data is an increasingly important part of many machine learning algorithms.  See for example the following; Clustering \cite{DD5}, Principal Coordinate Clustering \cite{DD4}, Power Weighted Shortest Paths for Clustering Euclidean Data \cite{DD2} , and the many references cited there-in \cite {DD8}.  

Hence-forth, we will consider two high dimensional data clustering algorithms where the data is presented as a subset of $\mathbb{R}^{n}$ ($\mathbb{R}^{n}$ denotes  $n$-dimensional Euclidean space.  Unless specified, $\mathbb{R}^{n}$ will be endowed with the Euclidean metric).

Note that unless otherwise specified, our examples will be generated with $n$ $=$ $2$, that is the data will be a subset of $\mathbb{R}^{2}$. This is because it is difficult to visualize and graph higher dimensions. However, all the algorithms in this paper will be discussed for any $n$ $>$ $1$.  
If any high dimensional data (more than two dimensions) is graphed, the data will be projected coordinate-wise into its first two dimensions.

Two algorithms which are discussed in this paper are the following:
\begin{enumerate}
  \item The first algorithm is the $s$-well separated pair decomposition ($s$-wspd) and realization which was introduced in reference \cite{DD1}.  Here, $s$ $>$ $0$ is a fixed separation parameter. See Section 2.1.
  \item The second algorithm is a $K$ nearest neighbors (KNN) algorithm using a power weighted shortest path metric (p-wspm) which was introduced in reference \cite{DD2}. Here, $K$ $>$ $1$ is a fixed integer. See Section 4. 
\end{enumerate}

Each algorithm will now be discussed in terms of its purpose, its parameter(s) and its dependencies.  We will also give examples of each algorithm implemented on different sets of data.
\section{s-Well Separated Pair Decomposition and Realization}
For $s$ $>$ $0$, this section will introduce the $s$-wspd and   give an overview of how its data structure is constructed.  The next section will discuss the algorithm's data structure in term of how to construct a realization.  This is analogous to saying that the next section will discuss how to obtain $s$-well separated pairs from the $s$-wspd data structure.
\subsection{s-Well Separated Pair Decomposition}
The purpose of this algorithm is to create a data structure that organizes data points in pairs of clusters in a way that preserves approximate distance according to some predefined metric.  This algorithm clusters the data as a hierarchical approach\cite{DD9}.  As mentioned before, we will work with the Euclidean metric in this paper.  We will call this data structure the $s$-well separated pair decomposition data structure (for short, $s$-wspd data structure).  Note that we will distinguish the $s$-wspd algorithm from the $s$-wspd data structure by explicitly saying 'data structure' when referring to the data structure.  The $s$-wspd data structure is a binary tree-like structure that allows for the comparison of two nodes in a certain tree to compute a realization.  The data in each node of the $s$-wspd is a single point set.  Each point set must satisfy \textbf{Definition 2.1} (shown below).  

\begin{definition}\label{Definition 2.1}
Let S correspond to the point set in node $N$ with the following conditions
\begin{enumerate}
  \item[(a)] $S$ $\subseteq$  $\mathbb{R}^{n}$ , $n > 1$.  
  \item[(b)] $|S|$ $> 0$. Here, $|S|$ denotes the cardinality of the set $S$.
  \item[(c)] $S$ is a set of unique elements. 
  \end{enumerate}
\end{definition}

(a)-(c) in \textbf{Definition 2.1} are not strong assumptions that the set $S$ needs to satisfy.  Below are some examples of point sets that do and do not satisfy the conditions in \textbf{Definition 2.1}.

\vspace{5mm}
Point sets that do satisfy \textbf{Definition 2.1}:

\begin{itemize}
  \item \{(1, 1), (2, 2), (3, 3)\}
  \item \{(1, 2), (2, 1)\}
  \item \{(1, 1, 1, 1, 1)\}
\end{itemize}
\vspace{0mm}

Point sets that do not satisfy \textbf{Definition 2.1}:

\begin{itemize}
  \item \{(1), (2)\} - fails condition (a).
  \item \{\} - fails condition (b).
  \item \{(1, 2, 3), (1, 2, 3)\} - fails condition (c).
\end{itemize}
\vspace{5mm}

At a high level, the $s$-wspd data structure is constructed by taking an input node, finding the dimension in which the data spans the largest length and splitting the previously mentioned dimension according to a few rules (These rules will not be described in this paper, rather they are described by Paul Callahan \cite{DD1}).  The result of the split is two point-sets, in which two new child nodes are formed.  The data in each child node is one of the split point sets.  Then each newly constructed node undergoes the same process.  This is done recursively until the node has a point set which contains a single data point.  The pseudo-algorithm is given below.

\newpage
\begin{algorithm}
\SetKwData{Left}{left}\SetKwData{This}{this}\SetKwData{Up}{up}
\SetKwFunction{Union}{Union}\SetKwFunction{FindCompress}{FindCompress}
\SetKwInOut{Input}{input}\SetKwInOut{Output}{output}
\Input{A node which the point set contains all data points, AKA the head of the tree}
\Output{A valid $s$-wspd}

\addvspace{16pt}
\emph{construct\_s\_wspd(Node N)}:\\
\textbf{begin} 
\BlankLine
largDim = findLargestDimention($N$)\;
Node leftChild, Node rightChild = split($N$, largDim)\;
construct\_s\_wspd(Node leftChild)\;
construct\_s\_wspd(Node rightChild)\;
\BlankLine
\textbf{end} 
\caption{Constructing a $s$-wspd data structure}\label{algo_disjdecomp}
\end{algorithm}\DecMargin{1em}

Note that when there are a lot of data points, the recursion depth can become quite large.  Hence, this algorithm can be implemented by using a deque to avoid a stack overflow.  This algorithm implemented with a deque is a trivial extension of \textbf{Algorithm 1}, so it will not be shown.
\footnote{Also note that the data in each node can be perfectly reconstructed by taking the union of the data in each of the node's children.}

The realization is discussed below. 

\newpage
\subsection{Realization}

We now need to understand the concept of minimum $(n-1)$-spheres (also called minimum bounding spheres) before we discuss the realization. Given an integer $n > 1$, a minimum $(n-1)$-sphere is a sphere in $\mathbb{R}^{n}$ that bounds a set of points in $\mathbb{R}^{n}$.  
This sphere bounds the points while having minimum radius.  

Hence, in when dealing with data in $\mathbb{R}^{3}$, $n = 3$ hence $n-1 = 2$.  We will thus have a 2-sphere when dealing with data in $\mathbb{R}^{3}$.  

\textbf{Figure 1} below shows a $(n-1)$-sphere in 2 dimensions that is not a minimum $(n-1)$-sphere because the radius is not minimized.  Note, this is a 1-sphere since the data it is trying to bound is in $\mathbb{R}^{2}$.

\begin{figure}[h]
\includegraphics[scale = .33]{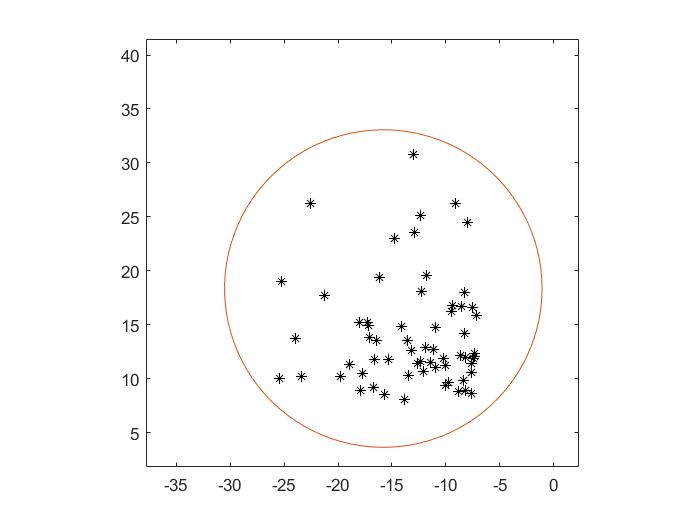}
\caption{1-sphere without the minimum radius 
}
\end{figure}

\textbf{Figure 2} below shows a $(n-1)$-sphere in 2 dimensions that is not a minimum $(n-1)$-sphere since it does not bound the data points. Note, this is a 1-sphere since the data it is trying to bound is in $\mathbb{R}^{2}$.

\begin{figure}[h]
\includegraphics[scale = .33]{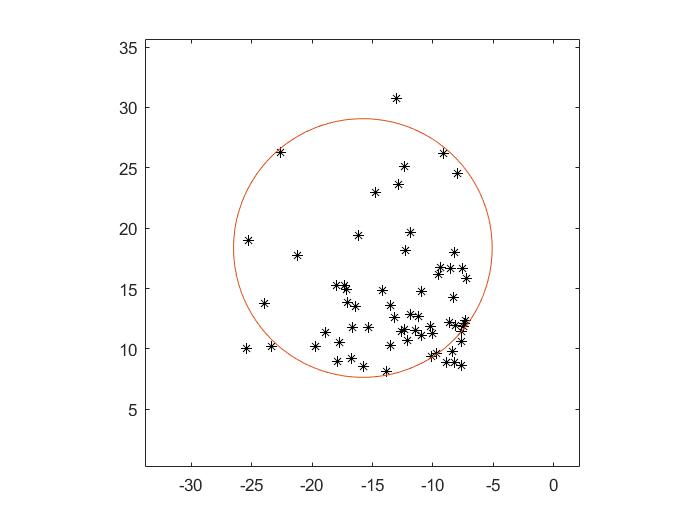}
\caption{1-sphere without the data points bounded in the $(n-1)$-sphere
}
\end{figure}

\newpage
\textbf{Figure 3} below shows the valid minimum $(n-1)$-sphere in 2 dimensions (AKA a 1-sphere)
\begin{figure}[h]
\includegraphics[scale = .33]{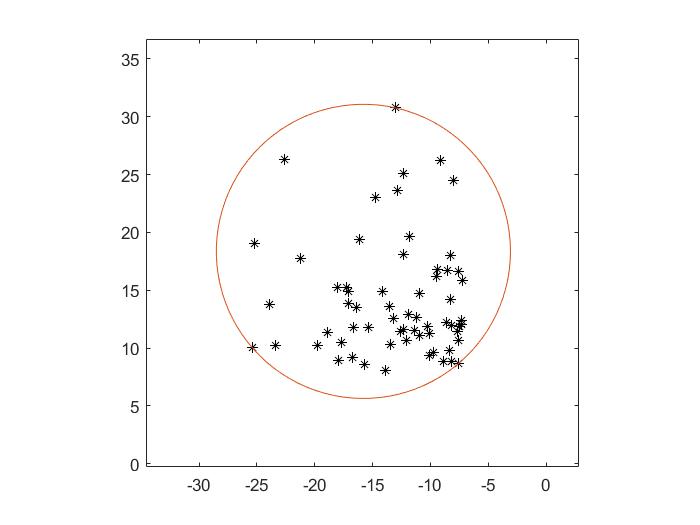}
\caption{1-sphere that bounds all data points in $\mathbb{R}^{2}$.
}
\end{figure}

We must also introduce the concept of $s$-well separated point-sets.  We assume we have two disjoint point-sets, and the valid minimum $(n-1)$-sphere around each point set.  The two point-sets are $s$-well separated if and only if the minimum distance between the two $(n-1)$-spheres is greater than $s$ multiplied by the maximum radius of the two $(n-1)$-spheres.  The variable $s$ acts as a separation parameter.  The formal definition is given below, as well as visuals.

\begin{definition}\label{Definition 2.2}
Lets assume two points sets $S_0$ and $S_1$, with each point set satisfying \textbf{Definition 2.1}.
Let us assume $S_0 \cap S_1 = \emptyset$.  Let $\rho$$0$ be the radius of the minimum $(n-1)$-sphere of $S_0$.  Let $\rho$$1$ be the radius of the minimum $(n-1)$-sphere of $S_1$.
Let $\varphi$ be the minimum distance between the two $(n-1)$-spheres.
Let $s$ be the separation parameter, defined by the user.
If the following condition is satisfied, $S_0$ and $S_1$ are $s$-well separated.
$s$ $*$ $\max$($\rho$$0$, $\rho$$1$) $\leq \varphi$.  Note, $*$ denotes multiplication.
\end{definition}

\begin{figure}[h]
\includegraphics[scale = .5]{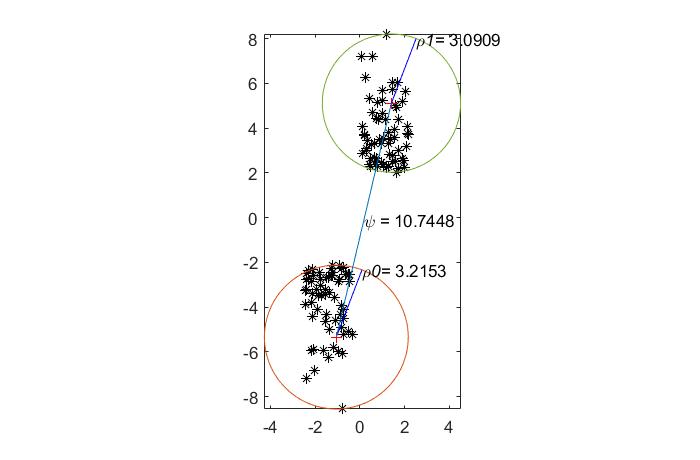}
\caption{\\
A $s$-well separated pair of non-empty point-sets.  $s$ = 1.5, dim = 2.\\
Note, $1.5 * \max(3.2153, 3.0909) \leq 10.7448$.
}
\end{figure}
\newpage

\begin{figure}[h]
\includegraphics[scale = .38]{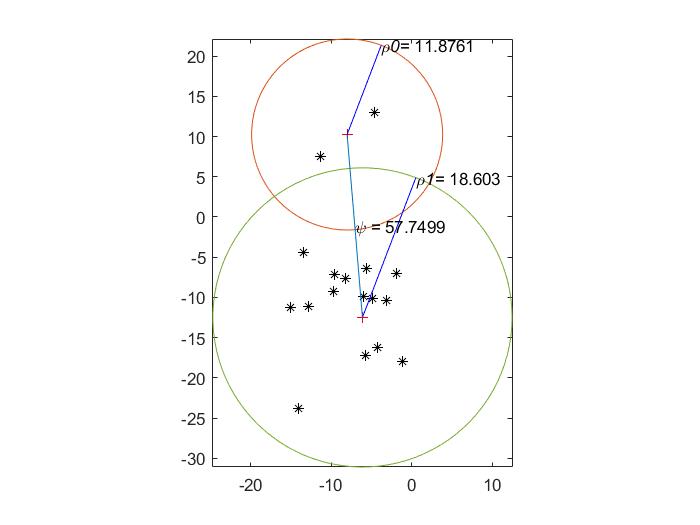}
\caption{\\
A $s$-well separated pair of non-empty point-sets. $s$= 1, dim=6.\\
Note, $1 * \max(11.8761, 18.603) \leq 57.7499$.
}
\end{figure}

Now that we can determine if two point-sets are $s$-well separated, we can use the $s$-wspd data structure to find all $s$-well separated point-sets.  This is analogous to computing the realization of the $s$-wspd data structure.  
To compute the realization of the $s$-wspd, we must have two non intersecting nodes.  Another way to say this is that none of the data points in node 0 should be in node 1, without loss of generality.  The point set corresponding to each node must also be non-empty.

This algorithm allows us to take any collection of points satisfying \textbf{Definition 2.1} and compute this $s$-well separated pair decomposition.  Once this $s$-well separated pair decomposition is constructed, we can recurse the tree, and find all possible $s$-well separated pairs that satisfy \textbf{Definition 2.2}.

Note, the algorithm computes the highest order well separated pairs.  That is, if set $S_0$ is $s$-well separated with set $S_1$, then all non-empty subsets of set $S_0$ are also $s$-well separated with all non-empty subsets of $S_1$.  The proof is below, and it holds without loss of generality.
	
\begin{lemma}
Let us assume we have two $s$-well separated point sets in Euclidean space, denoted as $S_0$ and $S_1$ respectively.  All  non-empty subsets of $S_0$ are $s$-well separated with $S_1$.  This holds without loss of generality.   
\end{lemma}
\begin{proof}
Let us assume that we have $S_0$ and $S_1$ from \textbf{Lemma 2.3}. We can denote the radius of the $(n-1)$-sphere of $S_0$ as $\rho_0$ and the radius of the $(n-1)$-sphere of $S_1$ as $\rho_1$.  We can also denote the minimum distance between the $(n-1)$-sphere of $S_0$ and the $(n-1)$-sphere of $S_1$ as $\varphi$. As stated before, $S_0$ and $S_1$ are $s$-well separated.
\\We can note that any subset of $S_0$ will have an $(n-1)$-sphere with a radius less than or equal to the radius of the $(n-1)$-sphere of $S_0$.  We will denote the radius of any arbitrary subset of the $(n-1)$-sphere of $S_0$ as $\rho_s$. As previously mentioned,  $\rho_s \leq \rho_0$.  
\\We can also note that the distance from the $(n-1)$-sphere of $S_0$ to the $(n-1)$-sphere of $S_1$ will always be greater than or equal to the distance from any arbitrary subset of the $(n-1)$-sphere of $S_0$ to the $(n-1)$-sphere of $S_1$. We will denote the distance between any arbitrary subset of the $(n-1)$-sphere of $S_0$ and the $(n-1)$-sphere of $S_1$ as $\varphi_s$.  As previously mentioned, $\varphi_s \geq \varphi$
\\
Since we know $S_0$ and $S_1$ are $s$-well separated, we also know that the following statement is true:  $s$ $*$ $\max$($\rho_0$, $\rho_1$) $\leq \varphi$\\
When evaluating if an arbitrary subset of $S_0$ is $s$-well separated with $S_1$, the following condition must hold: $s$ $*$ $\max$($\rho_s$, $\rho_1$) $\leq \varphi_s$\\
Since we know $\varphi_s \geq \varphi$ and $\rho_s \leq \rho_0$.  , the condition is always satisfied
\end{proof}

Overall, this allows us to generate $s$-well separated point-sets from any distribution.

\section{s-well Separated Pair Decomposition and Realization Dependencies and Observations}
This process is powerful, however it has some dependencies.  This process assumes that the point-sets will be clustered about a $(n-1)$-sphere.  Most other clustering pattern would not be optimal given the chosen realization.  For example, given a set of data points generated from a Cauchy distribution \cite{DD3} (location parameter $= 2$, scale parameter $= 2$), we get the following graph for the raw data and an example of an $s$-well separated points set, $s = 1$.

\begin{figure}[h]
\includegraphics[scale = .38]{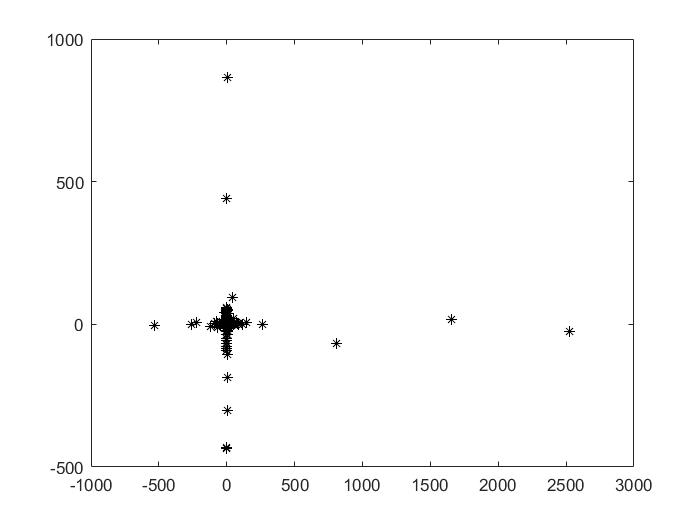}
\caption{
Raw Data generated from distribution Cauchy(2, 2).
}
\end{figure}

\begin{figure}[h]
\includegraphics[scale = .5]{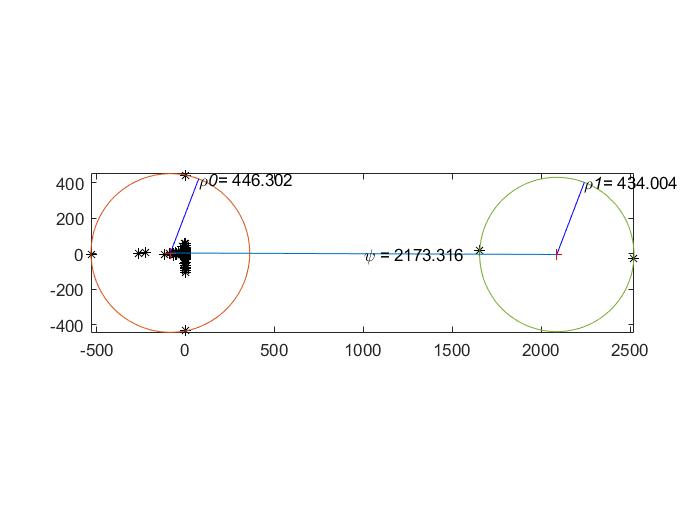}
\caption{A $s$-well separated point set from raw data in \textbf{Figure 6}, $s = 1$.
}
\end{figure}
\newpage
As seen, the clustering in \textbf{Figure 7} does NOT match the data, even though the two point-sets are $s$-well separated.

Additionally, this process also only compares two points sets at a time, effectively only clustering two points sets.  This algorithm would need non-trivial modifications to handle a case of three distinct clusters.

It is also important to note that the definition of clustering that this algorithm uses is fundamentally different from the definition of clustering other method's use. This method of clustering could create artificial clusters.  This is due to construction of the tree splitting the highest dimension, regardless of the point density.  This is oppose to the other definition of cluster, where high density regions are separated from low density regions. \cite{DD11} .An example of this is shown below, by imposing the clusters on the raw data.

\includegraphics[scale = .5]{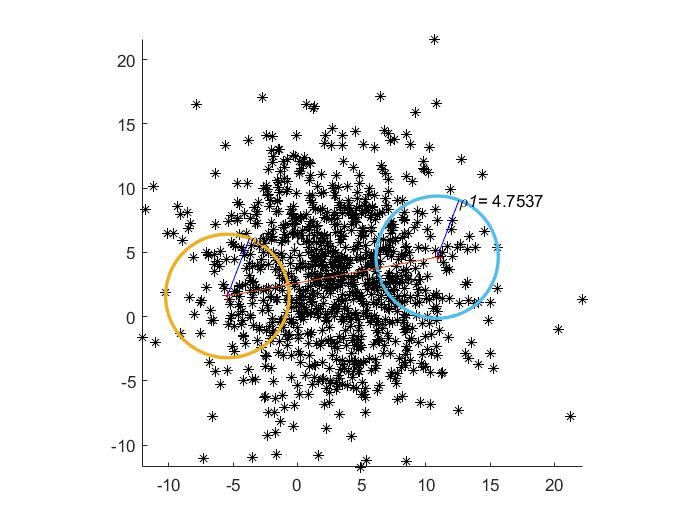}

\section{Power weighted Shortest Path Metric and Dependencies}
For an integer $K$ $>$ $1$, we consider an algorithm that is a permutation of Dijkstra's algorithm, that computes $K$-nearest neighbors (KNN) using a power weighted shortest path. This algorithm will form an arbitrary number of clusters of $K$ points.
This algorithm is developed in depth in \cite{DD2}, hence this paper will not go into it too deeply. 
This algorithm performs $n$-dimensional data clustering using a permutation of Dijkstra's algorithm to compute the $K$-nearest neighbors.  This algorithm utilizes the fact that it is assuming that the data is sampled from a higher dimensional manifold to prune unnecessary neighbors.  It does speed up the KNN algorithm greatly.  Results are shown below.

\begin{figure}[h]
\includegraphics[scale = .5]{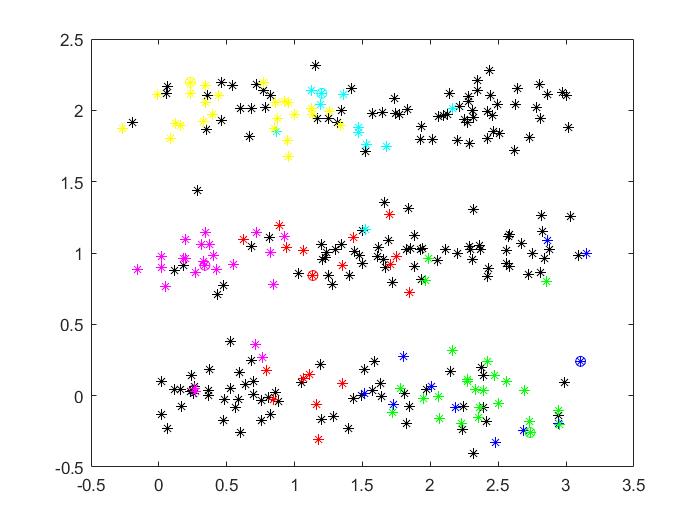}
\caption{The data set was generated in 100 ambient dimensions.\\Naive Dijkstra took 12.424 seconds, Dijkstra with pruning took 0.10195 seconds
}
\end{figure}

\clearpage
\section{Algorithm fusion}
In this section, we are now proposing to fuse the two algorithms discussed in the above sections.  The fused algorithm could have interesting implications on certain types of data sets.  It has already been shown that the k-nearest neighbors problem can be solved given the decomposition. \cite{DD6} \cite{DD7}
However, we will use the decomposition as a preprocessing step. then the processed data is fed into the p-wspm algorithm.  The preprocessing of the data greatly reduces the 'search space' of the second algorithm.  Hence, this is a two step clustering approach, which is commonly seen in various other clustering techniques \cite{DD10}. Discussed below are preliminary tests that were run to begin the algorithm fusion analysis.  The general process for the tests is as follows:
\begin{enumerate}
  \item We must first generate raw data to analyze.  To do this, we must determine the number of raw data points desired, the number of ambient dimensions, and the metric we want the raw data to be generated against.  Currently, we are able to generate raw data from any distribution supported by the c++ $<$random$>$ library.
  \item Once we have a generated data set, we must decide on the$s$value that will determine how well-separated each point set it.  This is discussed in \textbf{Section 2.2}.
  \item Once we have these parameters, we can construct the $s$-wspd data structure.  This is discussed in \textbf{Section 2.1}.
  \item Once the $s$-wspd data structure is constructed, we can recurse the $s$-wspd data structure to find all pairs of point-sets that are $s$-well separated.
  \item We can arbitrarily choose any pair of $s$-well separated point sets to use in the power weighted shortest path metric (p-wspm) algorithm that is discussed in \textbf{Section 4}.  In the following examples in \textbf{Section 6}, we chose the pair of $s$-well separated point-sets that have the greatest number of data points.
  \item We use the data points in the pair of $s$-well separated point-sets chosen above when running the p-wspm.  When running the p-wspm, select the ambient dimension (which should be the same as the value selected in the $s$-wspd algorithm), number of nearest neighbors to find, the number of clusters to produce, and the power weighting.
\end{enumerate}

\vspace{5mm}

In summary, we must select the following parameters for each of the algorithms:
\begin{itemize}
  \item $s$-well Separated Pair Decomposition Construction and Realization:
   \begin{itemize}
      \item The number of data points.
      \item The ambient dimension.
      \item The distribution the data is generated from.
      \item The $s$ value required to make the pair of point-sets $s$-well separated, AKA the separation parameter.
  \end{itemize}
  \item Power Weighted Shortest Path Metric:
    \begin{itemize}
      \item The number of clusters to form.
      \item The number of nearest neighbors to find.
      \item The power weighting.
  \end{itemize}
\end{itemize}

The next sections are examples of this process.

\newpage

\section{Visuals/Results}

\begin{figure}[H]
\includegraphics[width=0.5\textwidth]{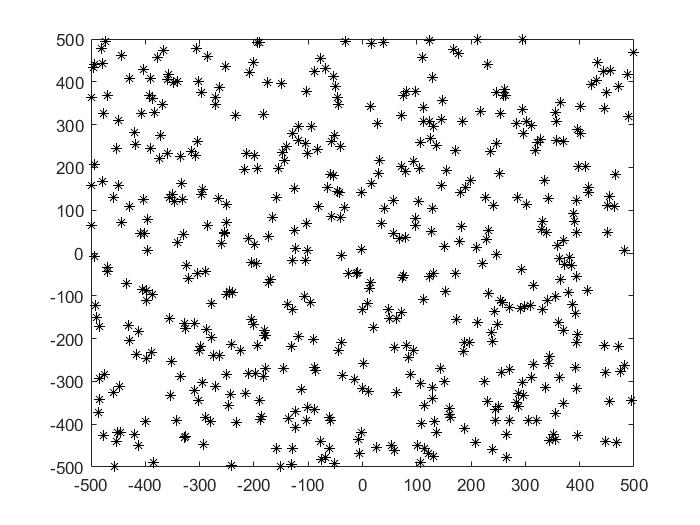}
\caption{$s$-wspd data structure construction parameters;\\There are 500 data points from a \\Uniform distribution (min = -500, max = 500) in $\mathbb{R}^{6}$ 
}
\end{figure}

\begin{figure}[H]
     \centering
     \begin{subfigure}[b]{0.495\textwidth}
         \centering
         \captionsetup{width=.85\linewidth}
         \includegraphics[width=\textwidth]{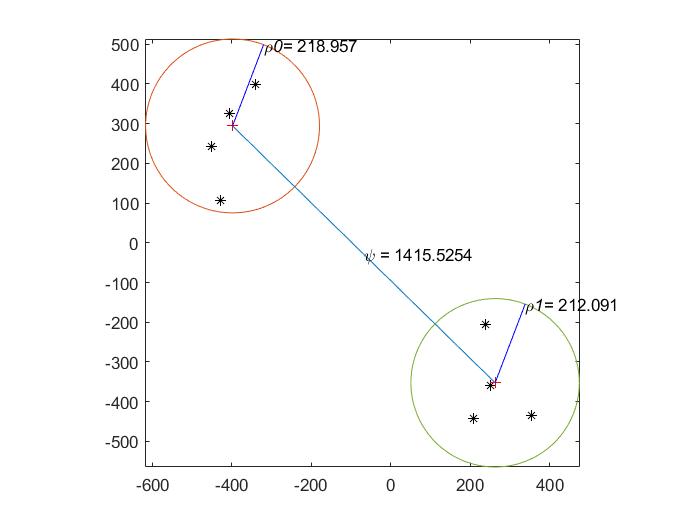}
         \caption{$s$-well separated point-sets parameters;\\The separation parameter 's',is chosen to be 4;$s$= 4.
}
     \end{subfigure}
     \hfill
     \begin{subfigure}[b]{0.495\textwidth}
         \centering
         \captionsetup{width=.85\linewidth}
         \includegraphics[width=\textwidth]{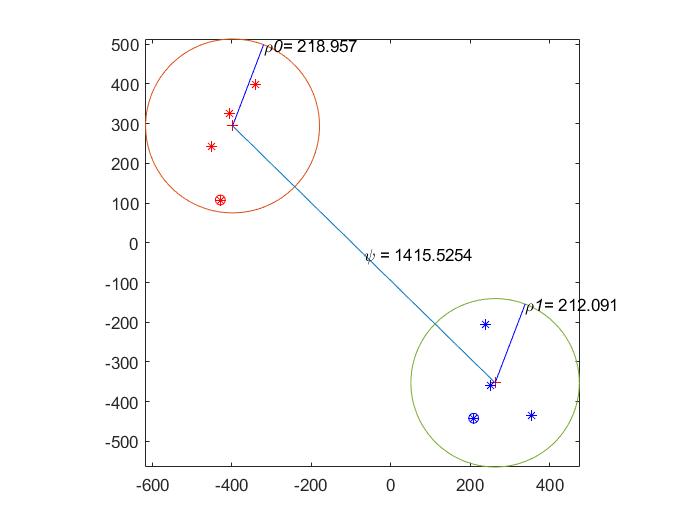}
        \caption{power-weighted shortest path metric parameters \# Clusters = 2, numNeighbors = 4, powerWeighting = 6.
}
     \end{subfigure}
\end{figure}
\newpage

\begin{figure}[H]
\includegraphics[width=0.5\textwidth]{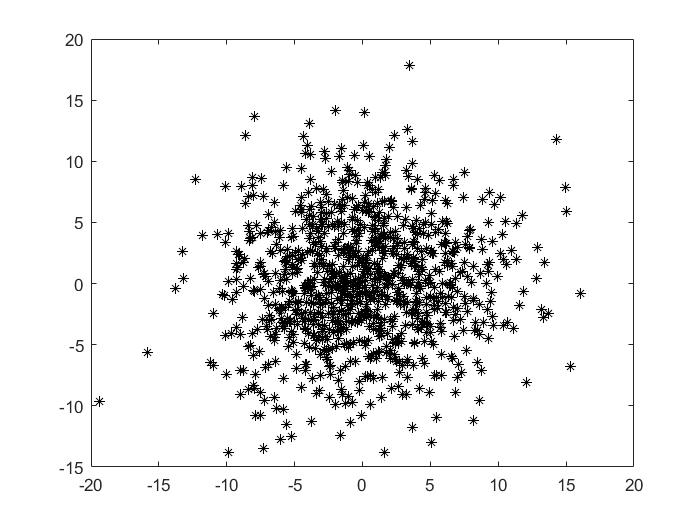}
\caption{$s$-wspd data structure construction parameters;\\There are 1000 data points from a \\Gaussian distribution (mean = 0, standard deviation = 5) in $\mathbb{R}^{6}$ 
}
\end{figure}

\begin{figure}[H]
     \centering
     \begin{subfigure}[b]{0.495\textwidth}
         \centering
         \captionsetup{width=.85\linewidth}
         \includegraphics[width=\textwidth]{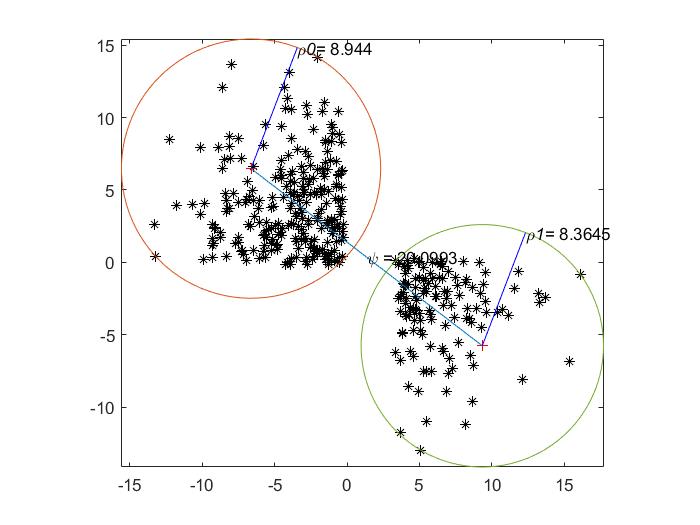}
    \caption{$s$-well separated point-sets parameters;The separation parameter 's', is chosen to be 0.25;$s$ = 0.25
    }
     \end{subfigure}
     \hfill
     \begin{subfigure}[b]{0.495\textwidth}
         \centering
         \captionsetup{width=.85\linewidth}
         \includegraphics[width=\textwidth]{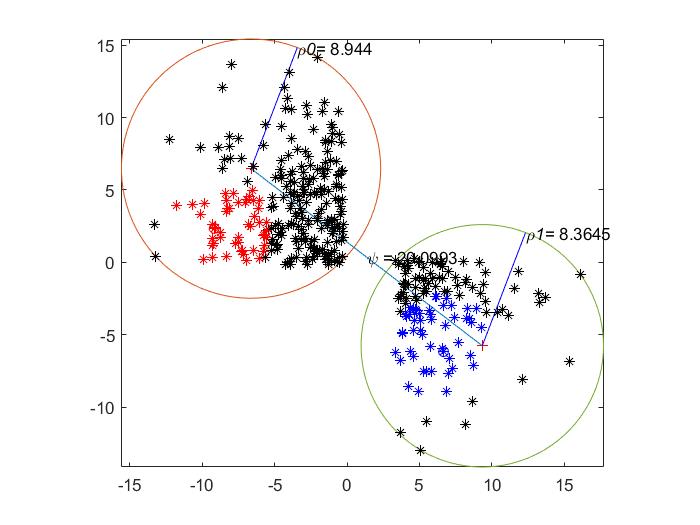}
    \caption{power-weighted shortest path metric parameters;numClusters=2; numNeighbors=25;  powerWeighting=2
    }
     \end{subfigure}
\end{figure}
\newpage

\begin{figure}[H]
\includegraphics[width=0.5\textwidth]{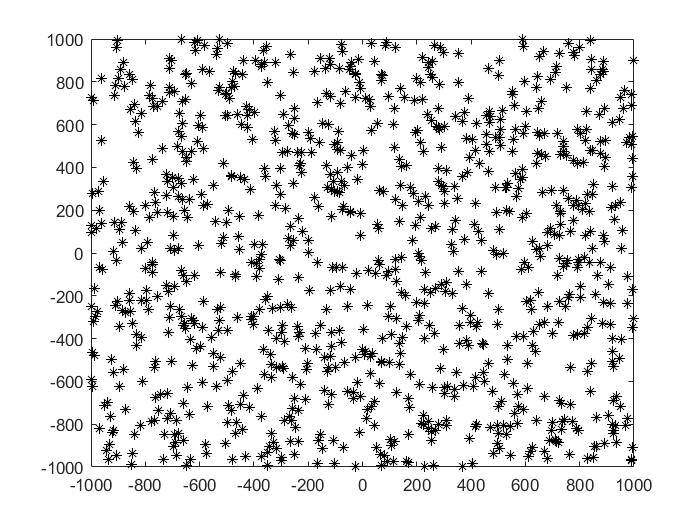}
\caption{$s$-wspd data structure construction parameters;\\There are 1000 data points from a \\Uniform distribution (min=-1000, max=1000) in $\mathbb{R}^{10}$ 
}
\end{figure}

\begin{figure}[H]
     \centering
     \begin{subfigure}[b]{0.495\textwidth}
         \centering
         \captionsetup{width=.85\linewidth}
         \includegraphics[width=\textwidth]{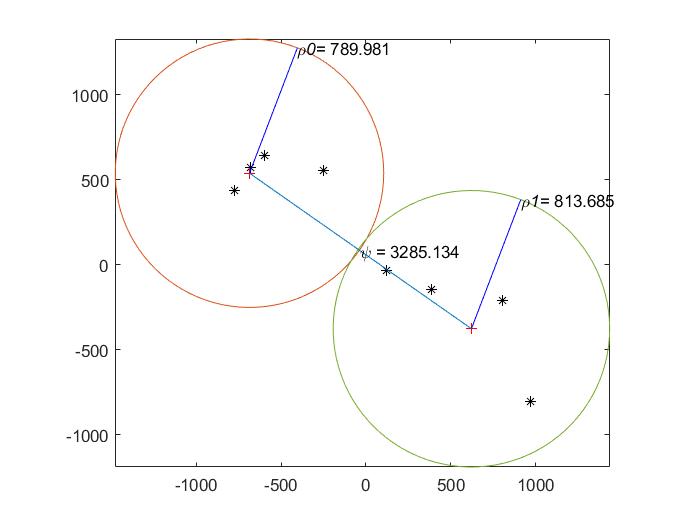}
    \caption{$s$-well separated point-sets parameters;The separation parameter 's', is chosen to be 2; $s$ = 2
    }
    \end{subfigure}
    \hfill
     \begin{subfigure}[b]{0.495\textwidth}
         \centering
         \captionsetup{width=.85\linewidth}
         \includegraphics[width=\textwidth]{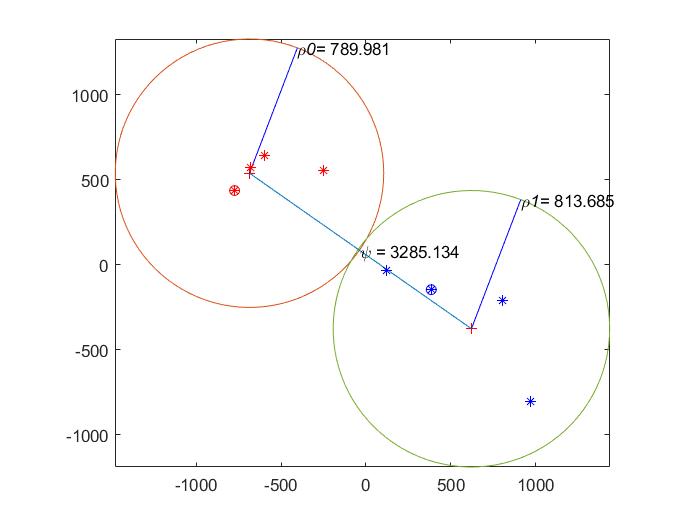}
\caption{power-weighted shortest path metric parameters;numClusters = 2; numNeighbors = 4; powerWeighting = 5
}
     \end{subfigure}
\end{figure}

\newpage

\section{Timing Comparisons}
Below shows the average time (averaged over 1,000 iterations) for each example shown in \textbf{Section 6}.  In each case, the time it took to perform the p-wspm on the data set is less with the preprocessing from the s-wspd algorithm.  Note that the table does not take into account the time it took to run the s-wspd algorithm, which was on the order of seconds.

Hence, it would be inefficient to preprocessing the data with the s-wspd algorithm if the p-wspm algorithm will only be run a few times.  If the p-wspm algorithm (or other clustering algorithms) will be run a lot, the preprocessing that the s-wspd algorithm does would be worth the time it takes.
\begin{table}[htbp]
    \centering
    \small
    \setlength\tabcolsep{2pt}

    \begin{tabular}{ c|c|c  }
        \multicolumn{3}{c}{Timings} \\
        \hline
        Algorithm parameters & \shortstack{p-wspm w/out  \\ pruning data}  & \shortstack{p-wspm w pruning data\\ with s-wspd algorithm} \\
        \hline
        \shortstack{Uniform (min = -500, max = 500),\\  $\mathbb{R}^{6}$, s = 4, numNeighbors=4, p=6}   & .025 secs   & .009 secs\\
        \hline
        \shortstack{Uniform (min = -500, max = 500),\\  $\mathbb{R}^{6}$, s = 4, numNeighbors=4, p=6}   & .0133 secs    & .0103 secs\\
        \hline
        \shortstack{Uniform (min = -500, max = 500),\\  $\mathbb{R}^{6}$, s = 4, numNeighbors=25, p=2}   & .005 seconds    &.003 secs\\
        \hline

    \end{tabular}

\end {table}

\newpage

\section{Future Research}
There are many future research topics that are useful to explore.  Three are given below: 

 \begin{itemize}
    \item \textbf{Determine what parameters of the $s$-wspd algorithm leads to the greatest efficiency in the p-wspm pruning process.}  This can be done by computing KNN with the naive Dijkstra's algorithm, and compare its time with the time using the modified Dijkstra's.
    \item \textbf{Extend the $s$-wspd algorithm to form a $(n-1)$-ellipse rather than a $(n-1)$-sphere when determining if 2 point sets are $s$-well separated.}  This will allow this algorithm to cluster a wider range of data.  Currently, the algorithm forms a bounding $(n-1)$-sphere around the data points, allowing it to cluster data clustered about a $(n-1)$-sphere. 
    \item \textbf{Extend the $s$-wspd algorithm to cluster an arbitrarily large number of clusters. Currently, this algorithm can only determine if two separate point sets are $s$-well separated.  This effectively only clusters two point sets, a limit on the algorithm.}  This algorithm can be modified to produce two $s$-well separated point sets, then take the larger point set and run the $s$-well separated algorithm again on it.
    \item \textbf{Investigating the fused algorithm on real data.  For example, studying the s-wspd algorithm for different data driven metrics.}  This is the simplest and important aspect of future research.  This can help extract certain properties of certain data sets. 
 \end{itemize}

\end{document}